\documentclass[accepted]{uai2026} 
                        
\usepackage[utf8]{inputenc}

\usepackage[american]{babel}

\usepackage{natbib} 
\usepackage{mathtools} 

\usepackage{amsfonts}
\usepackage{amssymb}
\usepackage{amsthm}
\usepackage{makecell}
\usepackage{graphicx}
\usepackage{subcaption}
\usepackage{placeins}
\usepackage{float}
\usepackage{comment}
\usepackage{cleveref}

\usepackage{tikz}
\usetikzlibrary{arrows.meta, positioning, shapes.geometric}
\usetikzlibrary{positioning}

\usepackage{booktabs} 
\usepackage{tikz} 

\newtheorem{theorem}{Theorem}[section]
\newtheorem{proposition}[theorem]{Proposition}
\newtheorem{lemma}[theorem]{Lemma}

\newtheorem{definition}[theorem]{Definition}
\newtheorem{assumption}[theorem]{Assumption}
\newtheorem{remark}[theorem]{Remark}
\newtheorem{example}[theorem]{Example}

\usepackage{multirow}

\newcommand{\GM}[1]{\textcolor{red}{\textbf{GM:} #1}}

\title{Conformal Graph Prediction with Z-Gromov-Wasserstein Distances}

\author[1,$\star$]{Gabriel Melo}
\author[1,$\star$]{Thibaut de Saivre}
\author[2]{Anna Calissano}
\author[1]{Florence d'Alch\'e-Buc}

\affil[1]{%
    LTCI, T\'el\'ecom Paris\\
    Intistut Polytechnique de Paris\\
    Palaiseau, France
}

\affil[2]{%
    Department of Statistical Science\\
    University College London\\
    London, UK
}

\affil[$\star$]{%
    Equal Contribution
}

\begin{document}
\maketitle
\begin{abstract}
Supervised graph prediction addresses regression problems where
the outputs are structured graphs. Although several approaches
exist for graph-valued prediction, principled uncertainty
quantification remains limited. We propose a conformal
prediction framework for graph-valued outputs, providing
distribution-free coverage guarantees in structured output spaces.
Our method defines nonconformity via the
Z-Gromov-Wasserstein distance, instantiated in practice
through Fused Gromov-Wasserstein (FGW),
enabling permutation invariant comparison between predicted
and candidate graphs.
To obtain adaptive prediction sets, we introduce Score Conformalized Quantile Regression (SCQR), an extension of Conformalized Quantile Regression (CQR) to handle complex output spaces such as graph-valued outputs. We evaluate the proposed approach on a synthetic task and a real problem of molecule identification.
\end{abstract}

\section{Introduction}

\begin{figure}
     \centering
     \includegraphics[width=0.95\linewidth]{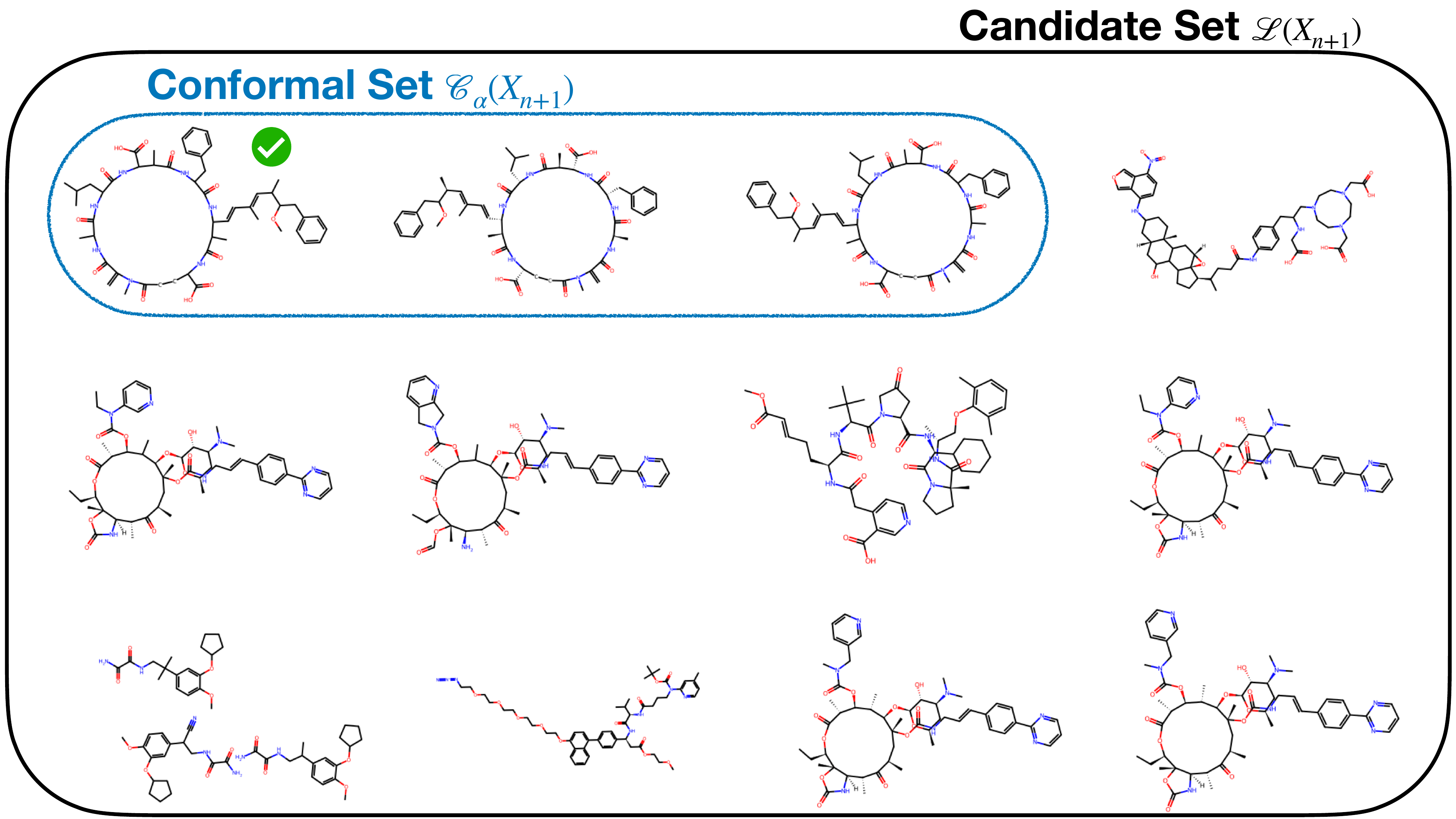}
     \caption{Conformal metabolite prediction set at $90\%$ marginal coverage for molecule identification. The green check-mark denotes the ground-truth.}
\label{fig:placeholder}
\end{figure}
Motivated by various applications such as molecular identification \citep{nguyen2019recent} in chemistry or scene understanding in computer vision \citep{shit2022relationformer}, Supervised Graph Prediction (SGP) has recently attracted a growing interest in Machine Learning. This task consists in learning a predictive model that maps an input variable of any modality (text, image, tabular or distributional) to a target graph of arbitrary size. Approaches in the literature leverage surrogate regression in graph representation spaces \citep{brouard2016fast} or end-to-end learning \citep{shit2022relationformer,yang2024exploiting, krzakala2024any2graph} that boils down to graph-valued regression \citep{calissano2022graph}. However, none of these provides confidence sets, posing risks for discovery or identification when experimental validation is costly.

Uncertainty quantification for graph-valued data aims to produce a set of plausible graphs rather than a single prediction. Although random graph models are well studied \citep{frieze2015introduction}, parametric approaches are often unrealistic in practice, especially for attributed graphs. Non-parametric methods therefore provide a more flexible alternative.

In this paper, we adopt Conformal Prediction (CP), an agnostic, distribution free, and post-training framework\footnote{The code is available at \url{https://github.com/gabrielmelo00/GraphConf}.}. CP \citep{shafer2008tutorial,fontana2023conformal} provides finite-sample coverage guarantees under minimal assumptions, requiring only exchangeability of the data, and can be applied on top of any pre-trained predictor without modifying training. 
Extending CP to graph-valued outputs raises several fundamental challenges. Even in the multivariate Euclidean setting, constructing informative conformal sets is known to be difficult due to the curse of dimensionality and the lack of natural orderings \citep{dheur2025unified,kondratyev2025neural,thurin2025optimal}. These challenges are amplified in the graph setting, where outputs live in a highly structured, non-Euclidean, and combinatorial space. 

For a graph with attributes on nodes and edges, uncertainty can revolve around the uncertainty at the global level, providing a conformal set made of a set of graphs with more or less different structure, or at a local level, measuring uncertainty at the level of the nodes and edges attribute. In this paper, we opt for a global framework. This choice is also supported by a specific focus on molecular identification tasks, where only a specific amount of configurations are possible.

A central challenge is therefore the choice of a suitable non-conformity score, which determines the structure and informativeness of the resulting conformal sets. Since graphs are defined up to node permutation, one must either work in the quotient space or use a permutation-invariant discrepancy. We adopt Optimal Transport distances, namely Gromov–Wasserstein and its variants, to define permutation-invariant scores and extend Conformal Prediction to graph-valued outputs with validity in the quotiented graph space.

We focus on Supervised Graph Prediction settings where each input is associated with a finite, input-dependent candidate set of graphs, as commonly assumed in structured output prediction and molecular identification.

Standard conformal prediction relies on a single global threshold, implicitly assuming homogeneous uncertainty across inputs. To account for input-dependent variability, we introduce \emph{Score Conformalized Quantile Regression} (SCQR), which calibrates conditional quantiles of the non-conformity score instead of a global cutoff. SCQR yields locally adaptive conformal sets while preserving marginal coverage guarantees.

Finally we empirically test our novel framework on a synthetic image-to-graph task and metabolite identification from mass spectra, assessed on Spectraverse, a recent benchmark in metabolomics \citep{gupta2026comprehensive}.

In short, our contributions include:
\begin{itemize}
\item a framework for Conformal Graph Prediction based on Z-Gromov-Wasserstein non-conformity scores and the proof of its validity in quotiented graph spaces;
\item a locally adaptive variant, Score Conformalized Quantile Regression
(SCQR), with proven marginal coverage; 
\item a set of numerical experiments showing the effectiveness and the versatility of the framework on a synthetic image-to-graph task and real molecule prediction problem.
\end{itemize}

\section{Related Work}
\paragraph{Prediction Models for Graphs.}
We consider supervised models whose outputs are graphs, including barycenter-based methods \citep{brogat2022learning,yang2024exploiting}, graph regression \citep{calissano2022graph}, deep end-to-end predictors \citep{krzakala2024any2graph,shit2022relationformer}, and graph-level autoencoders \citep{krzakala2025quest}. 
In molecular applications, graphs may also be predicted indirectly via SMILES representations \citep{zhang2025breaking}. 
Our framework is model-agnostic and applies to both direct graph predictors and SMILES-based pipelines. Unlike works that use graphs as input structures for node-, edge-, or graph-level prediction, our focus is on prediction problems whose outputs are themselves graphs.

\paragraph{Conformal Prediction for Graphs.}
Recent efforts have extended CP to non-Euclidean domains. For node-level tasks, \citet{huang2023uncertainty} introduced conformalized GNNs, establishing the importance of permutation invariance in base predictors, while \citep{lunde2025conformal} defined a node based CP sets. Related approaches construct conformal prediction sets for graph neural networks, including node classification and link prediction on fixed input graphs \citep{zargarbashi2023conformal,zargarbashi2024conformal}.
\citet{zhang2024conformal} proposed a general framework
for conformal structured prediction, constructing
structured prediction sets via implicit representations,
e.g., directed acyclic graphs for hierarchical labels. CP has also been studied for knowledge graph (KG) embeddings, turning its raw plausibility scores into entity or predicate answer sets that provably cover the true answer at a user-specified confidence level \citep{zhu2025conformalized,zhu2025predicate,zhu2025certainty}.

Closer to our setting, \citet{calissano2024conformal} developed CP sets for populations of unlabeled graphs, utilizing quotient spaces to handle the lack of node correspondence. Our work extends these concepts to graphs with categorical \textit{attributes} on nodes. In addition, by adopting the Z-Gromov-Wassertein distance, we work in a permutation invariant setting. Thus, whereas prior graph-related CP methods mainly use graphs as inputs or define sets over discrete entity/predicate completions of a KG, our setting requires calibrated sets in an unordered, attributed graph-valued output space.

\section{Background}
We begin by reviewing the two central components of our framework:
conformal prediction and the $Z$-Gromov--Wasserstein ($Z$-GW)
distance, a metric on metric measure spaces, under which graphs
naturally appear as an example of discrete metric measure spaces.

\paragraph{Notation}
Let $\Sigma_n := \{ a \in \mathbb{R}^n_{+} : \sum_{i=1}^n a_i = 1 \}$ denote the probability simplex,
and let $
\sigma_n := \{ P \in \{0,1\}^{n \times n} : P\mathbf{1}_n = \mathbf{1}_n,\; P^\top \mathbf{1}_n = \mathbf{1}_n \}
$
be the set of permutation matrices, where $\mathbf{1}_n \in \mathbb{R}^n$ denotes the all-ones vector.
For finite sets $\mathcal X=\{x_1,\dots,x_n\}$ and $\mathcal Y=\{y_1,\dots,y_m\}$, any discrete
probability measures $\mu\in\mathcal P(\mathcal X)$ and $\nu\in\mathcal P(\mathcal Y)$ can be written as
$\mu=\sum_{i=1}^n a_i \delta_{x_i}$ and $\nu=\sum_{j=1}^m b_j \delta_{y_j}$ with
$a\in\Sigma_n$ and $b\in\Sigma_m$. We identify admissible couplings $\pi\in\Pi(\mu,\nu)$ with
nonnegative matrices $\pi\in\mathbb R_+^{n\times m}$ satisfying $\pi\mathbf{1}_m = a$, $\pi^\top \mathbf{1}_n = b$.
When both measures are uniform, i.e. $\mu(\{x_i\})=\nu(\{y_j\})=\tfrac1n$ for all $i,j$ and $n=m$,
the set $\Pi(\mu,\mu)$ coincides with the Birkhoff polytope.
\subsection{Conformal Prediction}

\begin{definition}[Exchangeability]
A sequence of random variables $(Z_1, \dots, Z_{n+1})$ taking values in a measurable
space $\mathcal{Z}$ is said to be \emph{exchangeable} if, for any permutation
$\sigma$ of $\{1, \dots, n+1\}$ and for any $(z_1, \dots, z_{n+1}) \in \mathcal{Z}^{n+1}$, $
\mathbb{P}(Z_1 = z_1, \dots, Z_{n+1} = z_{n+1})
=
\mathbb{P}(Z_{\sigma(1)} = z_1, \dots, Z_{\sigma(n+1)} = z_{n+1}).
$
\end{definition}

Conformal Prediction (CP) provides finite-sample, distribution-free guarantees on the coverage of prediction sets.
In the regression setting, we observe exchangeable pairs $(X_i, Y_i)$ and have a pre-trained base predictor $f_\theta: \mathcal{X}\to\mathcal{Y}$. The framework relies on a \emph{nonconformity score} function $s: \mathcal{X \times \mathcal{Y}\to \mathbb{R}}$  that measures the discrepancy between the target $y$ and the prediction $f_\theta(x)$. For standard real-valued regression, a typical choice is the absolute residual:
\begin{equation}
s(x,y) = | y - f_\theta(x) |.
\end{equation}
Given a held-out calibration set $\mathcal{D}_{\text{cal}} = \{(X_i, Y_i)\}_{i=1}^n$, we compute the scores $R = \{R_1, \dots, R_n\}$ where $R_i = s(X_i, Y_i)$. We then compute the adjusted empirical quantile:
\begin{equation}
\hat{q}_{1-\alpha} = \text{Quantile}\left(\{R_i\}_{i=1}^n, \frac{\lceil (n+1)(1-\alpha) \rceil}{n}\right).
\end{equation}
For a new input $X_{n+1}$, the conformal prediction set is defined as:
\begin{equation}
\mathcal{C}(X_{n+1}) = \{ y \in \mathcal{Y} : s(X_{n+1}, y) \le \hat{q}_{1-\alpha} \}.
\end{equation}
By exchangeability of the calibration and test samples, the nonconformity scores
$\{R_i\}_{i=1}^{n+1}$ with $R_i = s(X_i, Y_i)$ form an exchangeable sequence.
Consequently, the rank of $R_{n+1}$ among $\{R_i\}_{i=1}^{n+1}$ is uniformly distributed,
which implies the marginal coverage guarantee \citep{vovk2005algorithmic}
\[
\mathbb{P}\big(Y_{n+1} \in \mathcal{C}(X_{n+1})\big) \ge 1 - \alpha,
\]
without any assumptions on the data distribution or the accuracy of the predictor.
We refer to \cite{angelopoulos2023conformal, shafer2008tutorial} for more details on conformal prediction. 

\paragraph{Conformalized Quantile Regression}
To extend standard conformal prediction to better handle heteroscedasticity, Conformalized Quantile Regression (CQR) \citep{romano2019conformalized} improves the conformal framework by leveraging quantile regression \citep{koenker1978regression} to construct input-adaptive prediction intervals. Instead of relying on a fixed, globally calibrated residual, CQR employs two base models, $\hat{\psi}_{\alpha/2}(x)$ and $\hat{\psi}_{1-\alpha/2}(x)$, trained to estimate the lower and upper $(\alpha/2, 1-\alpha/2)$ quantiles of the conditional distribution $Y|X=x$. The nonconformity score is then defined as $s(x, y) = \max\{\hat{\psi}_{\alpha/2}(x) - y, y - \hat{\psi}_{1-\alpha/2}(x)\}$, which measures the signed distance of the target to the boundaries of the predicted interval. By computing the $(1-\alpha)$-th quantile of these scores, $\hat{q}_{1-\alpha}$, on a calibration set, the resulting prediction intervals $\mathcal{C}(x) = [\hat{\psi}_{\alpha/2}(x) - \hat{q}_{1-\alpha}, \hat{\psi}_{1-\alpha/2}(x) + \hat{q}_{1-\alpha}]$ achieve valid marginal coverage while dynamically adjusting their width to account for local uncertainty in the data. 

\subsection{Z-Gromov Wasserstein Distance}
To measure discrepancies between structured outputs such as graphs,
which will later serve as the basis for conformal scoring,
we work in the space of $Z$-networks equipped with the
$Z$-Gromov-Wasserstein distance.

\begin{definition}[Metric Measure Spaces]
    A metric measure space (mm-space) is defined as triple $(\mathcal{X}, d_\mathcal{X}, \mu_\mathcal{X})$, where $ (\mathcal{X}, d_\mathcal{X})$ is a compact metric space and $\mu_\mathcal{X}$ is a Borel probability measure on $\mathcal{X}$ with $\mu_\mathcal{X}(\mathcal{X})=1$ and full support, $\mathrm{supp}(\mu_\mathcal{X}) = \mathcal{X}$.
\end{definition}
\vspace{-0.2cm}
Gromov--Wasserstein (GW) distances were introduced to compare mm-spaces via optimal transport \citep{memoli2011gromov}. Subsequent work in the machine learning literature proposed several
variants to handle structured and attributed data, most notably Fused
GW (FGW) \citep{titouan2019optimal}. The $Z$-Gromov-Wasserstein (Z-GW) distance \citep{bauer2025z} unifies these
approaches by replacing the metric-valued distance function in classical
GW with a general measurable pairwise relation taking values in a metric
space $(\mathcal{Z}, d_\mathcal{Z})$, so that classical GW, FGW and other formulations arise as particular instances of this construction.

\begin{definition}[$Z$-networks]
Let $(\mathcal{Z}, d_\mathcal{Z})$ be a metric space.
A \emph{$Z$-network} is a triple
$(\mathcal{X}, \omega_\mathcal{X}, \mu_\mathcal{X})$, where
$\mathcal{X}$ is a measurable space,
$\mu_\mathcal{X}$ is a probability measure on $\mathcal{X}$, and
$\omega_\mathcal{X}:\mathcal{X}\times\mathcal{X}\to\mathcal{Z}$ is a
measurable function encoding pairwise relational information.
Unlike mm-spaces, $\omega_\mathcal{X}$ is not required to be a metric.
\end{definition}

\paragraph{Z-Gromov--Wasserstein distance.}
Given two $Z$-networks
$(\mathcal{X}, \omega_\mathcal{X}, \mu_\mathcal{X})$ and
$(\mathcal{Y}, \omega_\mathcal{Y}, \mu_\mathcal{Y})$,
the $Z$-Gromov--Wasserstein $p$-distance ($p\ge1$) is defined as
\begin{equation}
\label{eq:zgw}
\begin{aligned}
\mathrm{GW}^Z_p(&\mathcal{X}, \mathcal{Y})
=
\frac{1}{2}
\inf_{\pi \in \Pi(\mu_\mathcal{X}, \mu_\mathcal{Y})}
\Bigg(
\iint_{(\mathcal{X}\times\mathcal{Y})^2}
 \\
&d_{\mathcal Z}\!\Big(
\omega_\mathcal{X}(x,x'),\omega_\mathcal{Y}(y,y')
\Big)^p
\,\mathrm{d}\pi(x,y)\,\mathrm{d}\pi(x',y')
\Bigg)^{1/p}.
\end{aligned}
\end{equation}
When $\mathcal{Z}=\mathbb{R}$ and $\omega_\mathcal{X}$ is the usual metric distance ($\omega_\mathcal{X} = d_\mathcal{X}$),
$\mathrm{GW}^Z_p$ reduces to the classical GW distance.

\begin{definition}[Weak isomorphism \citep{bauer2025z}]
\label{def:weak-iso}
Two $Z$-networks
$(\mathcal{X}, \omega_\mathcal{X}, \mu_\mathcal{X})$ and
$(\mathcal{Y}, \omega_\mathcal{Y}, \mu_\mathcal{Y})$
are said to be \emph{weakly isomorphic}, written
$\mathcal{X} \sim \mathcal{Y}$, if there exists a $Z$-network
$(\mathcal{W}, \omega_\mathcal{W}, \mu_\mathcal{W})$ and
measure-preserving maps
$\phi_\mathcal{X} : \mathcal{W} \to \mathcal{X}$ and
$\phi_\mathcal{Y} : \mathcal{W} \to \mathcal{Y}$
such that
\[
\omega_\mathcal{W}(w,w')
=
\omega_\mathcal{X}(\phi_\mathcal{X}(w),\phi_\mathcal{X}(w'))
=
\omega_\mathcal{Y}(\phi_\mathcal{Y}(w),\phi_\mathcal{Y}(w'))
\] for $\mu_\mathcal{W}\otimes\mu_\mathcal{W}~\text{for almost every pair}~ (w,w').$
\end{definition}

\paragraph{Weak isomorphism and metricity.}
The distance $\mathrm{GW}^Z_p$ defines a pseudometric on the space of
$Z$-networks.
To obtain a genuine metric, one considers equivalence classes under
\emph{weak isomorphism}. We denote by $\mathfrak{M}$ the collection of all $Z$-networks
$(\mathcal{X}, \omega_\mathcal{X}, \mu_\mathcal{X})$, and by
$\mathcal{M} := \mathfrak{M}/\!\sim$ the corresponding quotient space
under weak isomorphism.

\begin{theorem}[{\citep[Thm.~29]{bauer2025z}}]
For any separable metric space $(\mathcal{Z}, d_\mathcal{Z})$ and
$p \ge 1$, the distance $\mathrm{GW}^Z_p$ induces a genuine metric on
$\mathcal{M}$.
\end{theorem}

\begin{remark}[Working with representatives in practice]
Although the quotient space $\mathcal{M} := \mathfrak{M}/\sim$ is the
natural mathematical domain on which $\mathrm{GW}^Z_p$ is well defined,
it is not directly accessible in practice.
Each element of $\mathcal{M}$ corresponds to an equivalence class of
isomorphic realizations, and any concrete representative in
$\mathfrak{M}$ is inherently arbitrary.
In practical settings, observed data are therefore provided as explicit
$Z$-networks in $\mathfrak{M}$ (e.g., graphs, attributed graphs, or
meshes), and learning algorithms operate on these
representations.
Consequently, both empirical datasets and prediction tasks are
formulated at the level of $\mathfrak{M}$, while permutation-invariant
losses such as $\mathrm{GW}^Z_p$ ensure consistency with the underlying
quotient structure.
\end{remark}
\section{Conformal Graph Prediction}
\label{sec:conformalgraph}

\subsection{Problem Setup}
\begin{figure}
    \centering
   \includegraphics[width=0.8\linewidth, trim=19 45 30 24, clip]{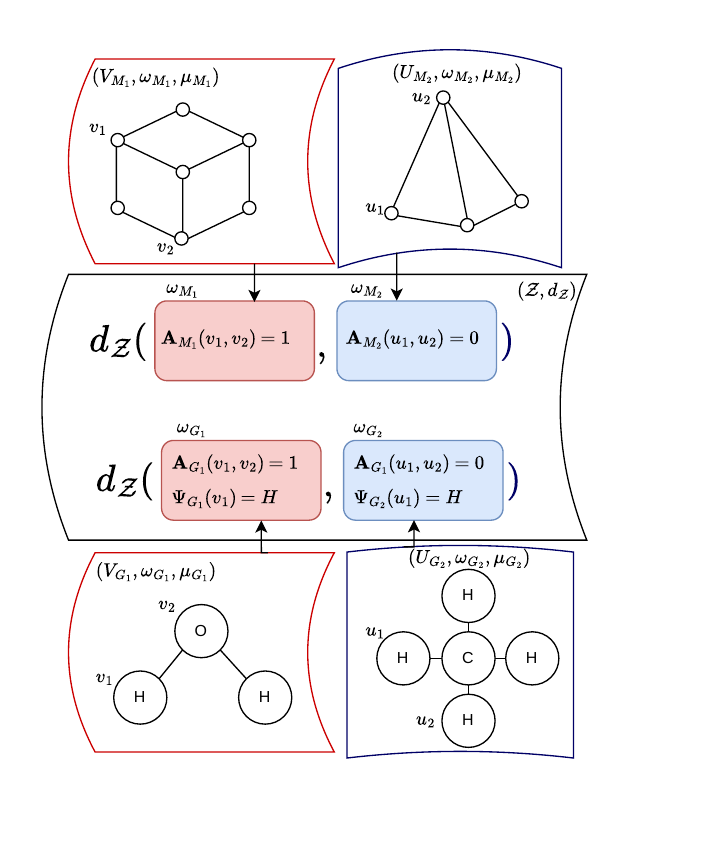}
   \vskip-0.4cm
    \caption{Z-GW distance}
    \label{fig:placeholder}
\end{figure}
\paragraph{Graphs as $Z$-networks.}
Graphs can be naturally viewed as finite $Z$-networks.
Let $G = (V,E,F)$ be a graph, where
$V$ is the node set with $|V| = n \in \mathbb{N}^*$,
$E \subseteq V \times V$ is the edge set, and
$\Psi : V \to \mathbb{R}^d$ is a node feature map, with $\Psi(i)$ denoting
the feature associated to node $i \in V$.
We associate to $G$ a $Z$-network $(V,\omega_G,\mu_G)$, where
$\mu_G \in \mathcal P(V)$ is a probability measure on $V$ (assumed uniform in this case), and
$\omega_G : V \times V \to \mathcal Z$ encodes pairwise relational
information between nodes.
This representation provides a unified way to encode both graph
structure and node attributes, enabling the comparison of graphs within
the general Z-GW distance.

\begin{example}[Fused Network Gromov-Wasserstein Distance] 
Choosing the pairwise structure $\omega_G$ such that it combines structural
and node-level information allows to retrieve the so-called Fused Network Gromov-Wasserstein Distance introduced by \citet{yang2024exploiting}:
\[
\omega_G(i,k)
=
\big(
\mathbf{A}(i,k),
\; \mathbf{X}(i,k),  \Psi(i)
\big)
\in \mathcal Z := \Omega \times \mathbb{R}^m \times \mathbb{R}^d,
\]
where the matrix $\mathbf{A} : V \times V \to \Omega$ defined from the set of edges $E$ encodes graph
connectivity (e.g. Adjacency, Shortest Path, Laplacian, etc), $\mathbf{X} : V \times V \to \mathbb R^m$ represents edge features, and $\Psi(i) \in \mathbb{R}^d$ denotes the feature associated to node $i$. For instance, when $\mathbf{A}$ is the adjacency matrix, $\Omega = \{0,1\}$. The product space $\mathcal Z$ is equipped with
a weighted $\ell_q$ metric
\begin{align}
    d_{\mathcal Z}&\big((a, x, \psi),(a',x',\psi')\big)
=
\Big(
\beta d_\Omega(a,a')^q\notag\\ 
&+
\gamma d_{\mathbb R^m}(x,x')^q
+
(1-\gamma -\beta) d_{\mathbb R^d}(\psi,\psi')^q
\Big)^{1/q},
\end{align}
for  $\gamma, \beta \in [0,1],\; q \ge 1$. \notag
\end{example}

\begin{example}[Fused Gromov-Wassertein Distance \citep{vayer2020fused}]
  Setting $\gamma=0$ and $\beta\in[0,1]$, we recover the Fused Gromov-Wasserstein (no more edge features). 
\end{example}

\begin{example}[Gromov-Wasserstein Distance \citep{memoli2011gromov}] To retrieve Gromov-Wasserstein distance between unlabeled graphs (no more feature nodes), we set $\gamma=0$ and $\beta=1$.
\end{example}

\paragraph{Discrete $Z$-Gromov--Wasserstein for graphs.}
Let $G_1=(V_1,\omega_1,\mu_1)$ and $G_2=(V_2,\omega_2,\mu_2)$ be finite
graphs with $|V_1|=n$ and $|V_2|=m$ and uniform measures.
The discrete $Z$-GW distance between $G_1$ and $G_2$ is defined as
\begin{equation}
    \min_{\pi\in\Pi(\mu_1,\mu_2)}
\left(
\sum_{i,k\in V_1}
\sum_{j,l\in V_2}
d_{\mathcal Z}\!\big(\omega_{1}(i,k),\omega_{2}(j,l)\big)^p
\,\pi_{ij}\pi_{kl}
\right)^{1/p}.
\end{equation}

In the finite graph setting with uniform measures, weak isomorphisms admit
a simple characterization (see Proposition \ref{prop:weak_iso_perm} below). As a consequence, $\mathrm{GW}^Z_p$ defines a permutation-invariant
distance on graphs. 

\begin{proposition}[Weak isomorphism and permutation invariance]
\label{prop:weak_iso_perm}
Let $(V_1,\omega_1,\mu_1)$ and $(V_2,\omega_2,\mu_2)$ be finite
$Z$-networks associated to $G_1$ and $G_2$, respectively, with $|V_1|=|V_2|=n$ and uniform measures
$\mu_1=\mu_2$.
Then $G_1\sim_n G_2$ if and only if there exists a permutation matrix
$P\in\sigma_n$ such that $\omega_1 = P^\top\omega_2 P$.
\end{proposition}

From now on, we specialize the spaces $\mathfrak M$ and $\mathcal M$ to
graphs.
We denote by $\mathfrak M$ the collection of all graphs with at most $N$
nodes, i.e. $\mathfrak M = \bigsqcup_{n=1}^N
\mathcal{G}_n$, in which $\mathcal{G}_n$ are the space of all graphs with $n$ nodes represented in a labeled form, and by $\mathcal M := \bigsqcup_{n=1}^N (\mathcal{G}_n/\!\sim_n)$, the
union of the quotient graph spaces of size $n$, where graphs are identified up to node
permutation. For simplicity, we note $\mathcal M = \mathfrak M /\!\sim$.

\paragraph{Graph-valued prediction.}
We consider a graph-valued prediction model \[
f_\theta:\mathcal X\to\mathfrak M
\] that maps inputs
$x\in\mathcal X$ (e.g., spectra, images, or other descriptors)
to graph-valued outputs $y\in\mathfrak M$. In the remainder, we assume that $f_\theta$ has been trained on a separate dataset $\mathcal{D}_{\text{train}}$  by minimizing a loss measuring discrepancy between predicted and
ground-truth graphs. This training set is independent of the calibration data $\mathcal{D}_{\text{val}}=\{(X_i,Y_i)\}_{i=1}^n$.
Crucially, such losses must be invariant to node relabeling; the
$\mathrm{GW}^Z_p$ distance, for instance, satisfies this requirement by construction. In the experimental part, we consider three off-the-shelves graph prediction models $f_\theta$: \textsc{Any2Graph}\citep{krzakala2024any2graph}, a general-purpose end-to-end graph predictor restricted to small-size graphs, and
\textsc{EmbCos}\citep{de2026small} and \textsc{MSAlign}\citep{krzakala2026msalign}, models specialized on molecule
identification from mass spectra.

\subsection{General Framework for Conformal Graph prediction}

Given the previous set up, for a given input $x$, we get a concrete graph $\hat y = f_\theta(x) \in \mathfrak M$. 
Uncertainty guarantees, however, must be invariant to node relabeling and are
therefore stated on the quotient space $\mathcal M$.

\begin{assumption}[Exchangeable data]
The observed data $(X_i, Y_i)_{i=1}^{n+1}$ are exchangeable and take values in
$\mathcal X \times \mathfrak M$.
\end{assumption}

\begin{definition}[Canonical projection]
\label{def:can_proj}
The canonical projection
$h : \mathfrak M \to \mathcal M := \mathfrak M / \sim$
maps a graph $y \in \mathfrak M$ to its equivalence class
\[
h(y) := [y]
= \{y' \in \mathfrak M \mid y' \sim y \},
\]
corresponding to all node relabelings of the same underlying graph.
\end{definition}

Given a graph predictor $f_\theta$ and representative-valued data in
$\mathfrak M$, we define the nonconformity score
\[
s:\mathcal X\times\mathfrak M\to\mathbb R_+,
\qquad
s(x,y)
:= \mathrm{GW}^Z_p\!\big(f_\theta(x),\,y\big).
\]

Since $\mathrm{GW}^Z_p$ is invariant under node permutation, the score
$s$ is invariant under the equivalence relation $\sim$.
As a consequence, $s$ factors through the canonical projection $h$ and
induces a well-defined score
\[
\tilde s:\mathcal X\times\mathcal M\to\mathbb R_+,
\qquad
\tilde s(x,\tilde y)=s(x, y),
\quad \tilde y=h(y),
\]
which depends only $\tilde y \in \mathcal M$.

\begin{lemma}[Exchangeability is preserved under quotient maps]
\label{lem:exchangeability_preserved}
Let $(X_i,Y_i)_{i=1}^n$ be an exchangeable sequence of random variables
taking values in $\mathcal X \times \mathfrak M$.
Let $h:\mathfrak M\to\mathcal M:=\mathfrak M/\sim$ be the canonical projection
onto equivalence classes, and define $\tilde Y_i := h(Y_i)$.
Then the induced sequence $(X_i,\tilde Y_i)_{i=1}^n$ taking values in
$\mathcal X \times \mathcal M$ is exchangeable.
\end{lemma}

Lemma~\ref{lem:exchangeability_preserved} ensures that the exchangeability of the
observed data is preserved when passing from representative-valued outputs in
$\mathfrak M$ to their equivalence classes in the quotient space $\mathcal M$.

We now establish conformal validity for $Z$-GW–based conformal sets.
Let $R_i = s(X_i,Y_i)$ be the calibration scores and let $\hat q_{1-\alpha}$ denote their empirical $(1-\alpha)$-quantile.
For a new input $X_{n+1}$, define
\begin{equation}
\label{eq:fgw-ball}
\mathcal{C}_\alpha^{\mathfrak M}(X_{n+1})
=
\Big\{
y \in \mathfrak{M} :
\mathrm{GW}^Z_p\big(f_\theta(X_{n+1}), y\big) \le \hat q_{1-\alpha}
\Big\}.
\end{equation} 

\begin{proposition}[Conformal validity on $\mathfrak M$ and $\mathcal M$]
\label{prop:conformal_validity_quotient}
Let $(X_i,Y_i)_{i=1}^{n+1}$ be exchangeable random variables taking values in
$\mathcal X\times\mathfrak M$, and let $f_\theta:\mathcal X\to\mathfrak M$ be a
predictor.
Define the nonconformity score
\[
s(x,y)
:= \mathrm{GW}^Z_p\!\big(f_\theta(x),\,y\big),
\]
and construct the conformal prediction set
$\mathcal C^{\mathfrak M}_\alpha(x)\subseteq\mathfrak M$ at level $\alpha\in(0,1)$ using
$s$.
Then the following hold:
\begin{enumerate}
\item (\emph{Marginal coverage})
\[
\mathbb P\!\left( Y_{n+1}\in \mathcal C^{\mathfrak M}_\alpha(X_{n+1})\right)\;\ge\;1-\alpha.
\]
\item (\emph{Well-defined on Quotient})
$\mathcal  C^{\mathfrak M}_\alpha(x)$ is a union of equivalence classes and therefore induces
a well-defined prediction set on the quotient space, $\mathcal C^{\mathcal M}_\alpha(x)
:= h\!\left(\mathcal C^{\mathfrak M}_\alpha(x)\right)\subseteq\mathcal M$,
which satisfies
\[
\mathbb P\!\left(h( Y_{n+1})\in \mathcal  C^{\mathcal M}_\alpha(X_{n+1})\right)\;\ge\;1-\alpha.
\]
\end{enumerate}
\end{proposition}
All proofs are deferred to Appendix~\ref{app:theory}.
\subsection{Practical Restrictions}

The conformal set defined in Eq.\ref{eq:fgw-ball} is an implicit subset of the graph space $\mathfrak{M}$, specified through a membership predicate rather than explicit enumeration. For any candidate graph $y \in \mathfrak{M}$, membership is determined by evaluating the nonconformity score $s(x, y)$ and comparing it to the calibrated threshold. In many graph domains such as molecular data, however, $\mathfrak{M}$ is combinatorially large, making the explicit materialization of $\mathcal{C}(x)$ computationally infeasible.

To obtain a finite and tractable prediction set, we intersect the implicit conformal set with an input-dependent candidate library $\mathcal{L}(x) \subset \mathfrak{M}$ (e.g., a metabolite database determined by a mass spectrum):
\begin{equation}
\label{eq:candidate}
\mathcal{C}_{\mathcal{L}}(x) = \mathcal{C}(x) \cap \mathcal{L}(x) = \{y \in \mathcal{L}(x) : s(x, y) \leq \hat{q}_{1-\alpha}\}.
\end{equation}
In retrieval tasks, the candidate library is typically constructed to be complete, in the sense that the ground-truth output is contained almost surely, i.e., $\mathbb{P}(Y_{n+1} \in \mathcal{L}(X_{n+1} )) = 1$.
Under this assumption, $\mathcal{C}_{\mathcal{L}}(x)$ coincides with the exact conformal prediction set restricted to the reduced output space $\mathcal{L}(x)$ and therefore inherits the marginal coverage guarantee.

\begin{remark}
\label{rem:incomplete}
If the ground-truth output $Y_{n+1}$ may lie outside $\mathcal{L}(X)$, the coverage
degrades according to the probability of library incompleteness: $
\mathbb{P}(Y_{n+1}  \in \mathcal{C}_{\mathcal{L}}(X_{n+1} ))
\;\geq\;
(1 - \alpha) - \mathbb{P}(Y_{n+1}  \notin \mathcal{L}(X_{n+1} ))$.
\end{remark}

\section{Locally Adaptive Graph Conformal Prediction}
\label{sec:scqr}

Previously, we extended conformal prediction to graph-valued outputs using a single global threshold on a nonconformity score. However, a global cutoff implicitly assumes homogeneous uncertainty across inputs. In practice, this assumption is rarely satisfied: some instances are intrinsically easy, while others are highly ambiguous. A single threshold therefore tends to produce overly conservative sets for easy inputs and risks undercoverage for difficult ones. 

In complex, high-dimensional, or structured output spaces (e.g.\ graphs, manifolds, or functional data), applying standard CQR directly in the output space $\mathfrak{M}$ is often computationally or conceptually intractable, since conditional quantiles are not naturally defined for structured objects. To provide distribution-free \emph{locally adaptive} guarantees in these settings, we introduce \emph{Score Conformalized Quantile Regression} (SCQR).

\paragraph{Score Conformalized Quantile Regression}
Given a point predictor $f_\theta: \mathcal{X} \to \mathfrak{M}$ and a nonconformity score $s(x, y) \in \mathbb{R}$, with $x\in\mathcal{X}$ and $y\in\mathfrak{M}$. SCQR relaxes the Graph conformal prediction one-size-fits-all assumption by allowing the threshold to depend on input-dependent attributes $\omega(x)\in \Omega$, which capture heteroscedasticity. These attributes may be the input itself (identity map) or derived quantities reflecting its complexity, e.g. size of candidate set or some model embeddings.

Concretely, we first train a quantile regression model $\psi: \Omega\to \mathbb{R}$ using the pinball loss to estimate $(1-\alpha)$-conditional quantile of $s(x,y)$ on  a training set. We again compute the adaptive residuals $E_i = s(X_i, Y_i) - \psi(\omega(X_i))$. We then compute the empirical quantile of the residuals: $\hat{q}_{1-\alpha} = \text{Quantile}\left(\{E_i\}_{i=1}^n, \frac{\lceil (n+1)(1-\alpha) \rceil}{n}\right)$. For a new input $X_{n+1}$, the conformal prediction set is:
 \begin{equation}
 \label{eq:scqr_set}
     \mathcal{C}(X_{n+1}) = \{ y \in \mathfrak{M} : s(X_{n+1}, y) \le \psi(\omega(X_{n+1})) + \hat{q}_{1-\alpha}\}.
 \end{equation}

 The residuals $E_i$ may take both negative and positive values. Since $s(x,y)\ge 0$ and $s(x,y)=0$ indicates perfect conformity, negative residuals correspond to examples that are easier than predicted given $\omega(x)$, while positive residuals indicate an underestimation of difficulty. Retaining both is essential for adaptive and non-conservative conformal calibration.

\begin{proposition}[Marginal Coverage Guarantee of SCQR]
\label{prop:scqr-validity}
Let $\{(X_i, Y_i)\}_{i=1}^{n+1}$ be exchangeable random variables, and let 
$\omega:\mathcal{X}\to\Omega$ be a deterministic feature map.
Assume that the nonconformity score $s:\mathcal{X}\times\mathfrak{M}\to\mathbb{R}$
and the quantile regression function $\psi:\Omega\to\mathbb{R}$
are fixed measurable functions, independent of the calibration sample.
Then the SCQR prediction set expressed in Eq. \ref{eq:scqr_set} satisfies
\begin{equation}
\mathbb{P}\!\left( Y_{n+1} \in \mathcal{C}(X_{n+1}) \right) \ge 1 - \alpha .
\end{equation}
\end{proposition}
\begin{remark}[One-Sided Nature of SCQR]Unlike the original CQR for real-valued regression which often produces two-sided intervals $[q_{low}, q_{high}]$, SCQR is inherently one-sided. This is because nonconformity scores are typically designed to be non-negative, where smaller values represent better fits.\end{remark}

\section{Numerical Experiments}
To illustrate the general applicability of our approach,
we consider both a synthetic graph prediction benchmark and
a real-world metabolite retrieval task.  We refer to Appendix \ref{app:exp} for more details on implementations and extra results.

\subsection{Synthetic Dataset}
We use the synthetic Coloring dataset \citep{krzakala2024any2graph},
consisting of pairs $(X_i, Y_i)$ where $X_i$ is an image encoding a
graph-coloring instance and $Y_i$ is its ground-truth graph. Nodes take one of four discrete colors
(blue, green, yellow, red) and edges indicate spatial neighboring.
The task is to recover the underlying graph from the image,
providing a controlled benchmark for image-to-graph prediction.

\paragraph{Candidate sets.}
For each test instance $(X_i, Y_i)$, we construct the candidate set $\mathcal{L}(x)$ by selecting all graphs in the test split that share the same node-type configuration as $Y_i$, i.e., the same number of blue, green, yellow, and red nodes. This yields a controlled identification problem where
structurally distinct graphs share identical node statistics.

\begin{figure}[H]
    \centering
 \includegraphics[width=1\linewidth,trim={0.9cm 0.8cm 0.9cm 0.9cm}, clip]{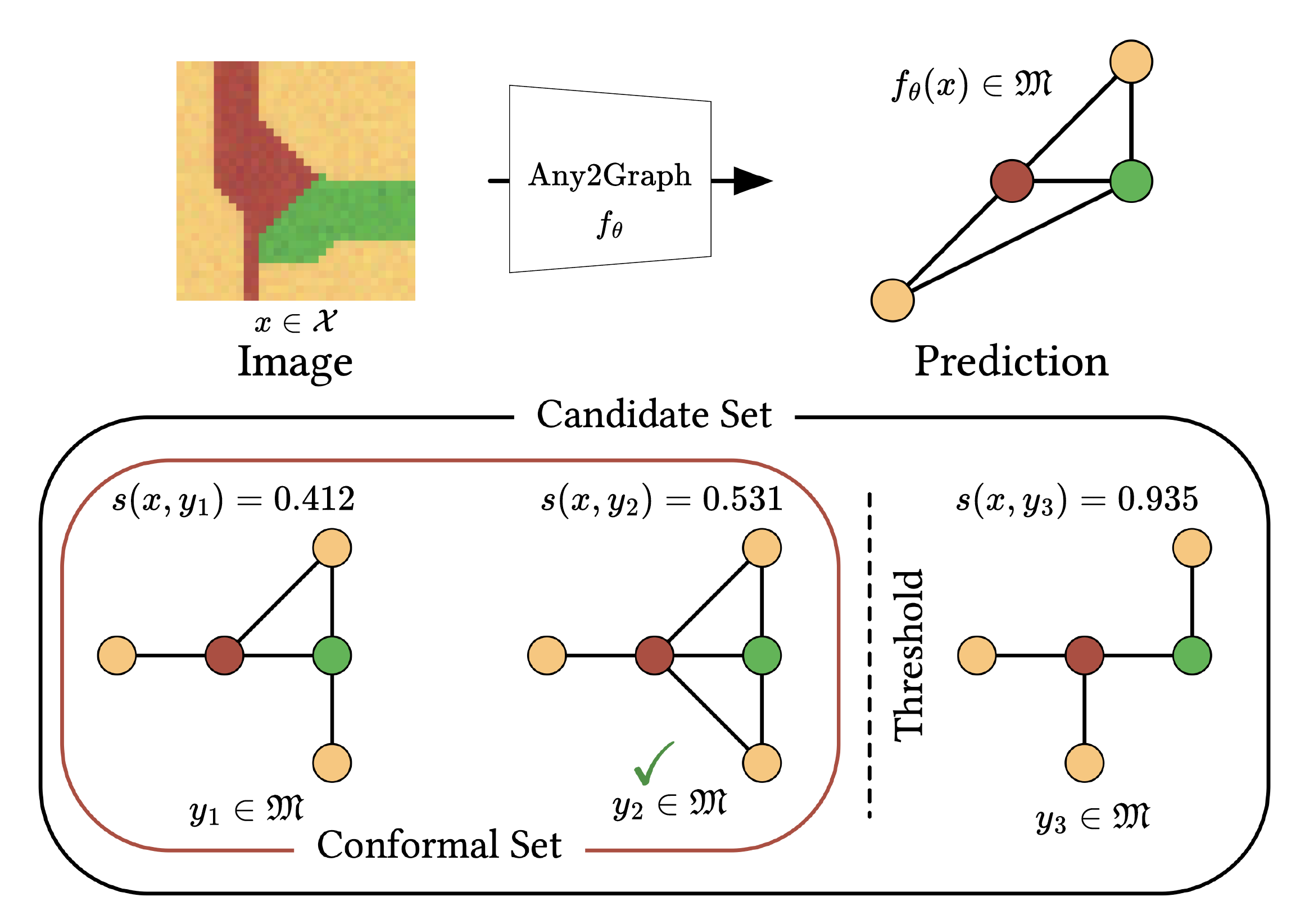}
   \vspace{-0.4cm}
    \caption{Example of the retrieval on the Coloring dataset. The green check-mark denotes the ground-truth.}
    \label{fig:conformal_example}
\end{figure}
\vspace{-0.6cm}

Figure~\ref{fig:conformal_example} illustrates the conformal prediction mechanism 
on a single Coloring instance. \textsc{Any2Graph} predicts a graph that does not 
exactly match the ground truth, yet the conformal set, constructed by thresholding 
FGW distances to all candidates, correctly contains it. The threshold, calibrated on 
held-out data, automatically adjusts to the predictor's error level. Candidates 
with FGW distance above the threshold are excluded, while the ground truth, being structurally close 
enough, remains inside the conformal set.

\paragraph{Model and protocol.}
We use \textsc{Any2Graph} \citep{krzakala2024any2graph} as graph
predictor $f_\theta$ (82\% test accuracy), trained with the
PMFGW loss. We use 100k samples for training and 10k each for
calibration and testing.

\subsection{Molecule Retrieval}
In many scientific applications, researchers aim to identify which molecules (also referred as metabolites) are
present in a biological or chemical sample. A standard experimental tool for
this task is mass spectrometry, which measures ionized molecules according to
their mass-to-charge ratio. In tandem mass spectrometry (MS/MS), a molecule is
further fragmented, producing a spectrum whose peaks provide information about
its underlying structure. The computational goal is then to retrieve, from a set
of candidate molecules, the molecular structure that most likely generated the
observed spectrum.

We evaluate our framework on MS/MS-based molecule identification using the
benchmark \emph{Spectraverse} \citep{gupta2026comprehensive}. Each input
$X_i \in \mathcal{X}$ is an MS/MS spectrum and $Y_i \in \mathfrak{M}$ is the
corresponding molecular graph. \emph{Spectraverse} contains spectra acquired
under different \emph{adduct types}, corresponding to different ionized forms of
a molecule, for example after gaining or losing a small charged group. Since
adducts can change the observed spectrum, each adduct type defines a distinct
dataset for evaluating our method. For this paper, each adduct type can simply be understood as a separate dataset of paired mass spectra and molecules, together with the associated candidate
molecules used for retrieval.

\paragraph{From spectra to graphs.}
We consider two existing methods for mapping spectra to molecular structures. 
\textsc{EmbCos}~\citep{de2026small} learns a shared embedding space for spectra 
and molecules from scratch, while \textsc{MSAlign}~\citep{krzakala2026msalign} 
leverages frozen pretrained encoders \textsc{DreaMS}~\citep{bushuiev2025self} and 
\textsc{ChemBERTa}~\citep{ahmad2022chemberta,chithrananda2020chemberta} for greater 
expressivity. Both output predicted SMILES strings (a standard text-based encoding 
of molecular structure), which are deterministically converted to molecular graphs 
via RDKit~\citep{landrum2013rdkit}, where nodes correspond to atom types. This 
yields the two-step pipeline $\text{MS/MS} \rightarrow \text{SMILES} \rightarrow 
\text{Graph}$, which we treat as a graph-valued predictor 
$f_\theta : \mathcal{X} \to \mathfrak{M}$.

\begin{remark}
Our focus is uncertainty quantification in a graph-valued output space, not improving 
MS/MS-to-graph prediction. We rely on the two-step pipeline above because direct 
MS/MS-to-graph approaches were not sufficiently accurate on this task.
\end{remark}

\paragraph{Candidate sets.}
For each spectrum $x$, we follow the same candidate library pipeline as \emph{MassSpecGym} \citep{bushuiev2024massspecgym}, which retrieves all molecules whose mass matches the observed spectrum $x$ and caps the set at $|\mathcal{L}(x)| \le 256$. All candidates are then converted to graphs via the same deterministic mapping, yielding a finite set $\mathcal{L}(x) \subset \mathfrak{M}$.

\begin{remark}
By construction of the candidates, the ground-truth molecule is always contained 
in $\mathcal{L}(x)$, so the completeness assumption $\mathbb{P}(Y_{n+1} \in 
\mathcal{L}(X_{n+1})) = 1$ is satisfied.
\end{remark}
\subsection{Experimental Protocol}
\newcommand{\adduct}[1]{\makecell[l]{Spectraverse\\\tiny{Adduct: #1}}}
\begin{table*}[t]
\centering
\caption{Comparison of conformal prediction (CP) and score conformalized quantile 
regression (SCQR) on graph-valued outputs with finite candidate sets. Acc denotes 
the base predictor retrieval accuracy. We report the metrics on non-empty predicted conformal sets, hence the changes in candidate set sizes.}
\label{tab:cp_scqr_comparison}
\scriptsize
\begin{tabular}{llcccccccccc}
\toprule
& & & & \multicolumn{2}{c}{\textbf{Conformal Set Size} $\downarrow$} & \multicolumn{2}{c}{\textbf{Candidate Set Size}} & \multicolumn{2}{c}{\textbf{Reduction} \% $\uparrow$} &
\\
\cmidrule(lr){5-6} \cmidrule(lr){7-8} \cmidrule(lr){9-10}
\textbf{Dataset} & \textbf{Model} $f_\theta$
& \textbf{Method}
& \makecell{\textbf{Emp.}\\\textbf{Cov.} (\%)}
& Mean & Median
& Mean & Median
& Mean & Median
& \makecell{\textbf{Empty}\\\textbf{rate} (\%) $\downarrow$} \\
\midrule
\multicolumn{11}{l}{\textit{Coverage $1-\alpha = 90\%$}} \\
\midrule
\multirow{2}{*}{Colors}
& \multirow{2}{*}{\makecell[l]{\textsc{Any2Graph} \\ \tiny{Acc: 82\%}}}
& CP                        & $90.2$ & $4$ & $1$ & $205$ & $223$ & $95.8\%$ & $98.9\%$ & $8.9\%$ \\
& & SCQR $|\mathcal{L}(x)|$ & $90.3$ & $4$ & $1$ & $201$ & $223$ & $95.6\%$ & $98.9\%$ & $8.8\%$ \\
& & SCQR $\operatorname{ResNet}(x)$ & $93.0$ & $4$ & $1$ & $205$ & $218$ & $96.2\%$ & $98.9\%$ & $5.7\%$ \\
\midrule
\multirow{3}{*}{\adduct{$[\text{M+CH}_3\text{COOH-H}]^{-}$}}
& \multirow{3}{*}{\makecell[l]{\textsc{MSAlign} \\ \tiny{Acc: 71.8\%}}}
& CP               & $90.0$ & $59$ & $56$ & $152$ & $159$ & $60.4\%$ & $61.0\%$ & $0.0\%$ \\
& & SCQR $|\mathcal{L}(x)|$ & $90.0$ & $47$ & $48$ & $147$ & $159$ & $64.5\%$ & $64.5\%$ & $0.0\%$ \\
& & SCQR $\textsc{DreaMS}(x)$ & $89.7$ & $29$ & $19$ & $151$ & $159$ & $79.4\%$ & $86.7\%$ & $0.7\%$ \\
\midrule
\multicolumn{11}{l}{\textit{Coverage $1-\alpha = 80\%$}} \\
\midrule
\multirow{3}{*}{\adduct{$[\text{M+CH}_3\text{COOH-H}]^{-}$}}
& \multirow{3}{*}{\makecell[l]{\textsc{EmbCos} \\ \tiny{Acc: 59.9\%}}}
& CP               & $84.7$ & $51.9$ & $53$ & $144.3$ & $133$ & $64.1\%$ & $63.5\%$ & $0.0\%$ \\
& & SCQR $|\mathcal{L}(x)|$ & $84.7$ & $52.1$ & $54$ & $144.3$ & $133$ & $63.9\%$ & $63.5\%$ & $0.0\%$ \\
& & SCQR $\textsc{DreaMS}(x)$ & $79.7$ & $27.8$ & $18$ & $144.1$ & $159$ & $80.0\%$ & $84.3\%$ & $0.0\%$ \\

\bottomrule
\end{tabular}
\end{table*}
\paragraph{Distance and nonconformity score.}Our nonconformity score is based on the $Z$-GW, instantiated in practice as FGW with task-specific cost matrices.
For Coloring, we use FGW with adjacency cost ($A$) and Feature Diffusion (FD) initialization of the solver. For Molecule Retrieval, we use FGW with Laplacian cost ($L$)
and identity initialization (Id).

\paragraph{Calibration.}
Standard conformal prediction (CP) estimates a single global $(1-\alpha)$ quantile 
from the calibration set, regardless of the input. Score Conformalized Quantile 
Regression (SCQR) instead conditions the threshold on an input-dependent attribute 
$\omega(x)$, allowing the prediction set size to adapt to properties of $x$. For 
the synthetic Coloring task, we consider two choices of $\omega(x)$: (i) the 
candidate set cardinality $\omega(x) = |\mathcal{L}(x)|$, and (ii) an image 
embedding obtained from a $\operatorname{ResNet}(x)$ variant~\citep{he2016deep}. For the molecule 
retrieval task, we similarly consider two variants: (i) the candidate set 
cardinality $\omega(x) = |\mathcal{L}(x)|$, and (ii) the learned spectral embedding 
$\omega(x) = \textsc{DreaMS}(x)$ of the input spectrum $x$, projected via Random 
Fourier Features~\citep{rahimi2007random}. Because the $\textsc{DreaMS}(x)$ embedding space is high-dimensional and only on the order of a thousand adduct-specific calibration points were available, directly regressing the conformal score in this space led to overfitting. We therefore apply random Fourier features (RFF), which regularizes the SCQR fit and enabled stable training despite the limited calibration data.

\paragraph{Evaluation metrics.}
All results are reported at nominal level $1-\alpha=0.9$.
We measure empirical coverage, conformal set size
(mean and median), reduction relative to the full candidate
library, and empty-set rate.

\subsection{Results}

\paragraph{Coverage Validity and Efficiency.}  Across both tasks, empirical coverage remains close to the nominal levels, confirming validity under $Z$-GW scoring, as reported in Table~\ref{tab:cp_scqr_comparison} and Fig.~\ref{fig:cp_scqr_hist}. We evaluate efficiency through conformal set size and reduction relative to the candidate library. On \textit{Coloring}, SCQR conditioned on $|\mathcal{L}(x)|$ performs similarly to CP: the median set size is $1$ for both, with over $95\%$ average reduction, and conditioning on the ResNet embedding brings only marginal further gains in mean reduction. On \textit{Molecule Retrieval}, conditioning matters more. SCQR based on $|\mathcal{L}(x)|$ offers only a modest improvement over CP, while conditioning on the spectral embedding $\textsc{DreaMS}(x)$ yields a much sharper reduction in conformal set size, roughly halving it relative to CP while keeping coverage close to nominal. The same pattern holds at the lower $80\%$ target with the weaker \textsc{EmbCos} base predictor: $|\mathcal{L}(x)|$-conditioned SCQR tracks CP closely, whereas $\textsc{DreaMS}(x)$-conditioned SCQR again produces markedly smaller sets.

\paragraph{Choice of conditioning attribute.}
Table~\ref{tab:scqr_conditioning} compares conditioning choices for SCQR on the 
Coloring dataset. While marginal coverage is achieved by all variants, the choice 
of $\omega(x)$ has a significant impact on conditional coverage. SCQR with image 
embeddings yields the most robust worst-slab coverage, outperforming candidate set size conditioning, suggesting that richer input representations 
better capture heteroscedasticity across instances.

\begin{table}[h]
\centering
\caption{Marginal and worst-slab coverage for different SCQR conditioning attributes 
on the Coloring dataset. Slabs group inputs by number of candidates.}
\label{tab:scqr_conditioning}
\scriptsize
\begin{tabular}{lcc}
\toprule
Method & Marginal Coverage & Worst-Slab Coverage \\
\midrule
CP                        & $90.2\%$ & $64.7\%$ \\
SCQR $|\mathcal{L}(x)|$  & $90.3\%$ & $67.6\%$  \\
SCQR $\operatorname{ResNet}(x)$        & $90.2\%$ & $73.5\%$ \\
\bottomrule
\end{tabular}
\end{table}

\begin{figure}[H]
    \centering\includegraphics[width=1\linewidth, trim={1cm 0.8cm 1cm 1cm}, clip]{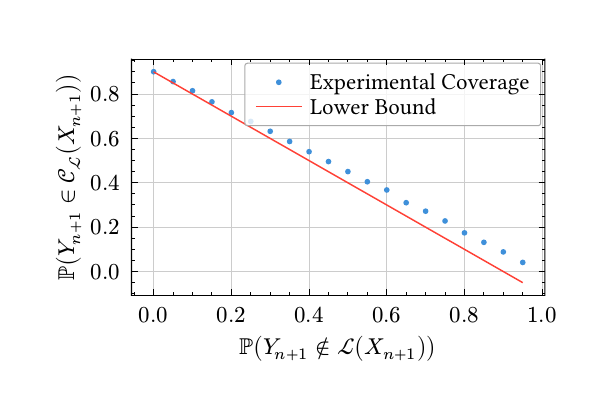}
    \vspace{-0.8cm}
    \caption{Empirical coverage as a function of the probability of ground-truth absence from the candidate library $\mathbb{P}(Y_{n+1} \notin \mathcal{L}(X_{n+1}))$, evaluated on the Coloring dataset. }
    \label{fig:missing}
\end{figure}

\paragraph{Candidate set incompleteness.}
When the assumption that the ground-truth output is always contained in the 
candidate library fails, Remark~\ref{rem:incomplete} provides a theoretical 
lower bound on coverage as a function of the incompleteness probability 
$\mathbb{P}(Y_{n+1} \notin \mathcal{L}(X_{n+1}))$. Figure~\ref{fig:missing} 
empirically validates this bound on the Coloring dataset by artificially removing 
the ground truth from the candidate library with increasing probability. Empirical 
coverage degrades gracefully and remains above the theoretical lower bound 
throughout.

\begin{figure}[H]
\centering
\begin{subfigure}[t]{0.48\linewidth}
    \centering
    \includegraphics[width=\linewidth]{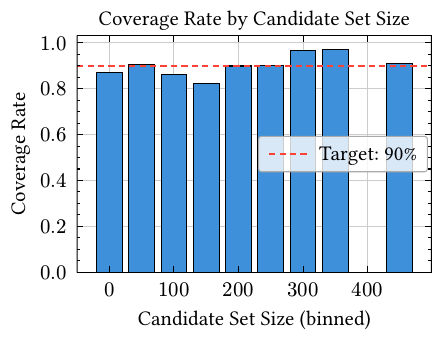}
    \caption{Coloring — CP}
\end{subfigure}
\hfill
\begin{subfigure}[t]{0.48\linewidth}
    \centering
    \includegraphics[width=\linewidth]{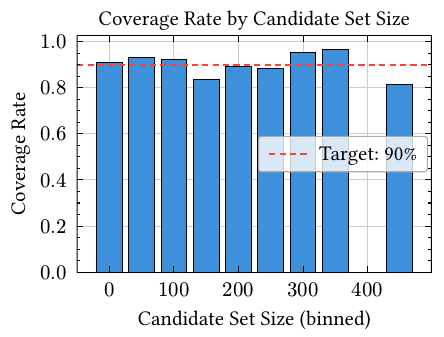}
    \caption{Coloring — SCQR}
\end{subfigure}

\vspace{1pt}

\begin{subfigure}[t]{0.48\linewidth}
    \centering
    \includegraphics[width=\linewidth]{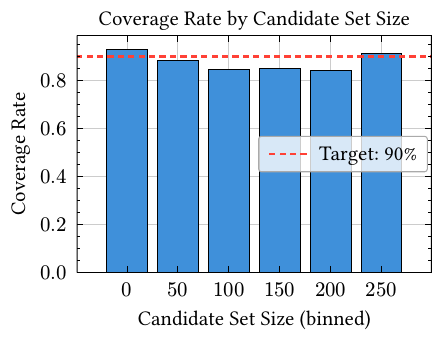}
    \caption{Molecule — CP}
\end{subfigure}
\hfill
\begin{subfigure}[t]{0.48\linewidth}
    \centering
    \includegraphics[width=\linewidth]{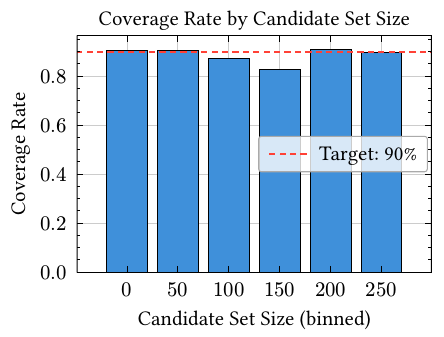}
    \caption{Molecule — SCQR}
\end{subfigure}

\caption{Empirical coverage versus candidate set
size, for CP (left) and SCQR (right). For SCQR, the attribute
function $\omega(x)$ is task-dependent: $\omega(x)=|\mathcal{L}(x)|$
(candidate set size) for Coloring, and $\omega(x)=\textsc{DreaMS}(x)$
(mass-spectrum embedding) for Molecule Retrieval.
}
\label{fig:cp_scqr_hist}
\end{figure}

\section{Conclusion}
We introduced a conformal prediction framework for
graph-valued outputs grounded in the $Z$-Gromov–Wasserstein
geometry. By defining nonconformity through $Z$-GW, the method provides
distribution-free coverage guarantees while respecting
the permutation-invariant structure of graphs.
We further proposed Score Conformalized Quantile Regression (SCQR),
which yields adaptive prediction sets by conditioning on
input-dependent signals, improving efficiency without
compromising validity. More broadly, the framework extends naturally to other metric for graph data and to other
structured output spaces representable as $Z$-networks,
such as meshes, point clouds, and distributions,
suggesting a general geometry-aware approach to
uncertainty quantification in structured prediction.

\begin{acknowledgements}
The authors thank Aymeric Dieuleveut for helpful discussions and feedback and Paul Krzakala for his insights on the experimental part of the work. The first author is funded  through the PEPR IA FOUNDRY (ANR-23-PEIA-0003) and the second author through the PRCE Far-see (ANR-24-CE23-0921). 

It also received funding from the European Union’s Horizon
Europe research and innovation programme under grant
agreement 101120237 (ELIAS). The author(s) would also like to thank the Isaac Newton Institute for Mathematical Sciences, Cambridge, for support and hospitality during the programme Representing, calibrating \& leveraging prediction uncertainty from statistics to machine learning, where work on this paper was undertaken. This work was supported by EPSRC grant EP/Z000580/1.
\end{acknowledgements}

\bibliography{bibliography}
\appendix
\onecolumn
\section{More on Theoretical Results}
\label{app:theory}
\subsection{Proof of Proposition \ref{prop:weak_iso_perm}}
\label{proof:weak_iso_perm}

\begin{proof}
Fix orderings $V_1=\{v_1,\dots,v_n\}$ and $V_2=\{w_1,\dots,w_n\}$ and identify
$\omega_1,\omega_2$ with their matrix representations.

\medskip
\noindent\emph{($\Rightarrow$)}
Assume $G_1\sim G_2$. By Definition~\ref{def:weak-iso}, there exist a finite $Z$-network
$W=(V_W,\omega_W,\mu_W)$ and measure-preserving maps
$\phi_1:V_W\to V_1$ and $\phi_2:V_W\to V_2$ such that
\[
\omega_W=\phi_1^{*}\omega_1=\phi_2^{*}\omega_2
\quad
\mu_W\otimes\mu_W\text{-a.e.}
\]

Since $\mu_1$ is uniform and measure preserving,
\[
\mu_W(\phi_1^{-1}(\{v\}))=\mu_1(\{v\})=\frac1n
\qquad\forall v\in V_1,
\]
hence $\phi_1$ is surjective; similarly for $\phi_2$.
Writing $\mu_W=\sum_{k=1}^N c_k\delta_{u_k}$ with $c_k>0$, the above implies
\[
\sum_{u_k:\,\phi_1(u_k)=v} c_k=\frac1n
\qquad\forall v\in V_1,
\]
and analogously for $\phi_2$.

Since $V_W$ is finite with full-support measure, the pullback identity holds
everywhere:
\[
\omega_W(u,u')=\omega_1(\phi_1(u),\phi_1(u'))
=\omega_2(\phi_2(u),\phi_2(u'))
\qquad\forall u,u'\in V_W.
\]

Choose a section $s:V_1\to V_W$ with $\phi_1\circ s=\mathrm{id}_{V_1}$ and define
$\varphi:=\phi_2\circ s:V_1\to V_2$. Then
\[
\omega_1(v_i,v_j)=\omega_2(\varphi(v_i),\varphi(v_j))
\qquad\forall i,j,
\]
so $\omega_1=\varphi^{*}\omega_2$. Since $\mu_1$ and $\mu_2$ are uniform atomic measures and $\varphi$ is
measure-preserving, each fiber $\varphi^{-1}(\{w\})$ has mass $1/n$.
As $|V_1|=|V_2|=n$, this implies that all fibers are singletons and hence
$\varphi$ is bijective. Let $P=(P_{ij})\in\sigma_n$ be the permutation matrix
defined by
\[
P_{ij} := \mathbb{I}_{\{\,w_j=\varphi(v_i)\,\}},
\]
so that $\omega_1 = P^\top \omega_2 P$.

\medskip
\noindent\emph{($\Leftarrow$)}
If $\omega_1=P^\top\omega_2 P$ for some $P\in\sigma_n$, let $\varphi$ be the
associated bijection and set $W:=G_1$, $\phi_1:=\mathrm{id}$,
$\phi_2:=\varphi$. Both maps are measure-preserving and the pullback identity
holds everywhere, hence $G_1\sim G_2$.
\end{proof}

\subsection{Extra results for Section \ref{sec:conformalgraph}}
\begin{lemma}[Measurable maps preserve equality in distribution]
\label{lem:measurable_preserve_law}
Let $(\Omega,\mathcal F,\mathbb P)$ be a probability space.
Let $Z$ and $Z'$ be random variables taking values in a measurable space
$(\mathcal Z,\Sigma_{\mathcal Z})$, and let
$g : \mathcal Z \to \mathcal Y$ be a measurable map into another measurable space
$(\mathcal Y,\Sigma_{\mathcal Y})$.
If
\[
Z \overset{d}{=} Z',
\]
then
\[
g(Z) \overset{d}{=} g(Z').
\]
\end{lemma}

\begin{proof} Let $B \in \Sigma_{\mathcal Y}$ be arbitrary.
Since $g$ is measurable, $g^{-1}(B) \in \Sigma_{\mathcal Z}$, and therefore
\[
\mathbb P\big(g(Z) \in B\big)
= \mathbb P\big(Z \in g^{-1}(B)\big)
= \mathbb P\big(Z' \in g^{-1}(B)\big)
= \mathbb P\big(g(Z') \in B\big).
\]
Since this holds for all measurable sets $B \subseteq \mathcal Y$, we conclude that
$g(Z)$ and $g(Z')$ have the same distribution.
\end{proof}

\subsection{Proof of Lemma \ref{lem:exchangeability_preserved}}
\label{proof-lem:exchangeability_preserved}
\begin{proof}
Exchangeability of $(X_i,Y_i)_{i=1}^n$ means that for any permutation
$\sigma$ of $\{1,\dots,n\}$,
\[
(X_1, Y_1,\dots,X_n,Y_n)
\;\overset{d}{=}\;
(X_{\sigma(1)}, Y_{\sigma(1)},\dots,X_{\sigma(n)}, Y_{\sigma(n)}).
\]

Define the measurable map $g:\mathcal X \times \mathfrak M \to \mathcal X \times \mathcal M$, 
$g(x,y) := (x,h(y))$.
Since $h$ is measurable, so is $g$.
Applying $g$ componentwise and invoking
Lemma~\ref{lem:measurable_preserve_law}, we obtain
\[
(X_1,\tilde Y_1,\dots,X_n,\tilde Y_n)
\;\overset{d}{=}\;
(X_{\sigma(1)},\tilde Y_{\sigma(1)},\dots,X_{\sigma(n)},\tilde Y_{\sigma(n)}),
\]
where $\tilde Y_i = h(Y_i)$.
Since this holds for all permutations $\sigma$, the sequence
$(X_i,\tilde Y_i)_{i=1}^n$ is exchangeable.
\end{proof}

\subsection{Proof of Proposition \ref{prop:conformal_validity_quotient}}
\label{proof-prop:conformal_validity_quotient}
\begin{proof}
By invariance of $s$ under $\sim$, it factors through the canonical
projection $h$ and induces a score on $\mathcal M$.
By Lemma~\ref{lem:measurable_preserve_law}, the score sequence
$\{s(X_i,Y_i)\}_{i=1}^{n+1}$ is exchangeable.
Standard conformal prediction theory \cite{angelopoulos2023conformal} then yields
\[
\mathbb P\!\left( Y_{n+1}\in \mathcal C^{\mathfrak M}_\alpha(X_{n+1})\right)\ge 1-\alpha.
\]

By construction, membership in $\mathcal C^{\mathfrak M}_\alpha(x)$ is characterized
by the inequality $s(x,y)\le \hat q_{1-\alpha}$. Since
$s(x,y)=\tilde s(x,h(y))$, this criterion depends on $y$ only through its
equivalence class $h(y)$: for any $y,y'\in\mathfrak M$ with $h(y)=h(y')$,
\[
y\in \mathcal C^{\mathfrak M}_\alpha(x)
\quad\Longleftrightarrow\quad
y'\in \mathcal C^{\mathfrak M}_\alpha(x).
\]

In other words, $\mathcal C^{\mathfrak M}_\alpha(x)$ is a union of equivalence classes, containing either all
node relabelings of a given graph or none of them. In particular, denoting by $h^{-1}(B):=\{y\in\mathfrak M : h(y)\in B\}$
the preimage of a set $B\subseteq\mathcal M$, 
\[
h^{-1}\!\big(\mathcal C^{\mathcal M}_\alpha(x)\big)
=\mathcal C^{\mathfrak M}_\alpha(x),
\qquad\text{where }
\mathcal C^{\mathcal M}_\alpha(x):=h\!\left(\mathcal C^{\mathfrak M}_\alpha(x)\right).
\]
Let $\tilde Y_{n+1}:=h(Y_{n+1})$; its law is the pushforward measure
$P_{\tilde  Y_{n+1}}=h_\# P_{Y_{n+1}}$, which is well defined since $\mathcal M$ is equipped with
the quotient $\sigma$-algebra, under which $h$ is measurable by construction.
Then
\begin{align*}
    \mathbb P\!\left(\tilde Y_{n+1}\in \mathcal C^{\mathcal M}_\alpha(X_{n+1})\right)
&=
\mathbb P\!\left( Y_{n+1}\in h^{-1}\!\left(\mathcal C^{\mathcal M}_\alpha(X_{n+1})\right)\right)\\
&=
\mathbb P\!\left( Y_{n+1}\in \mathcal C^{\mathfrak M}_\alpha(X_{n+1})\right),
\end{align*}
Combining with the coverage guarantee in
$\mathfrak M$ yields
$\mathbb P\!\left(\tilde Y_{n+1}\in \mathcal C^{\mathcal M}_\alpha(X_{n+1})\right)\ge 1-\alpha$.
\end{proof}

\subsection{Proof of Proposition \ref{prop:scqr-validity}}
\label{proof-prop:scqr-validity}

\begin{proof}
Define the residuals $E_i = s(X_i, Y_i) - \psi(\omega(X_i))$ for
$i=1, \dots, n+1$. Since $(X_i, Y_i)_{i=1}^{n+1}$ are exchangeable and the
functions $s$, $\omega$, and $\psi$ are measurable and independent of the
calibration data, the sequence $(E_1, \dots, E_{n+1})$ is exchangeable.

For $\beta\in(0,1)$, let $q_\beta(V_1,\dots,V_k)$ denote the
$\lceil \beta k\rceil$-th smallest value of $(V_1,\dots,V_k)$, so that the
threshold reads $\hat q_{1-\alpha} = q_{1-\alpha}(E_1,\dots,E_n,+\infty)$.
First note that
\[
E_{n+1} \le q_{1-\alpha}(E_1,\dots,E_n,+\infty)
\iff
E_{n+1} \le q_{1-\alpha}(E_1,\dots,E_n,E_{n+1}),
\]
since replacing $+\infty$ by $E_{n+1}$ leaves the
$\lceil (n+1)(1-\alpha)\rceil$-th smallest value unchanged whenever
$E_{n+1}$ lies above it, and both events hold whenever $E_{n+1}$ lies below
it. By exchangeability of $(E_1,\dots,E_{n+1})$, the event that $E_{n+1}$
is among the $\lceil (n+1)(1-\alpha)\rceil$ smallest of the $n+1$ residuals
has probability at least $\lceil (n+1)(1-\alpha)\rceil/(n+1)$, hence
\[
\mathbb{P}\left( E_{n+1} \le \hat q_{1-\alpha} \right)
\ge \frac{\lceil (n+1)(1-\alpha) \rceil}{n+1} \ge 1 - \alpha.
\]
Since $Y_{n+1} \in \mathcal{C}(X_{n+1}) \iff E_{n+1} \le \hat q_{1-\alpha}$,
the marginal coverage guarantee follows.
\end{proof}

\newpage
\section{Evaluation Metrics}
\label{app:metrics}

This appendix details the metrics reported in Table~\ref{tab:cp_scqr_comparison} for comparing conformal prediction (CP) and score-conditional quantile regression (SCQR) on graph-valued outputs with finite candidate sets.

\paragraph{Setup.}
For each test input $x$, let $\mathcal{L}(x)$ denote the associated candidate set of graphs, and let $\mathcal{C}_\alpha(x) \subseteq \mathcal{L}(x)$ be the conformal prediction set constructed at nominal coverage level $1-\alpha$. Let $y$ denote the ground-truth graph.

\paragraph{Empirical Coverage.}
Empirical coverage is computed as
\[
\widehat{\mathrm{Cov}}(1-\alpha)
=
\frac{1}{|\mathcal{D}_{\mathrm{test}}|}
\sum_{(x,y)\in\mathcal{D}_{\mathrm{test}}}
\mathbb{I}\big\{ y \in \mathcal{C}_\alpha(x) \big\}.
\]
Both CP and SCQR are evaluated at the same nominal level $1-\alpha$.

\paragraph{Prediction Set Size.}
The prediction set size for an input $x$ is given by $|\mathcal{C}_\alpha(x)|$.  
We report both the mean and the median of this quantity across the test set to account for potential skewness in the distribution of set sizes.

\paragraph{Relative Reduction.}
To quantify efficiency relative to the original candidate set, we define the relative reduction for each input as
\[
\mathrm{Reduction}(x)
=
\frac{|\mathcal{L}(x)| - |\mathcal{C}_\alpha(x)|}{|\mathcal{L}(x)|}
\times 100\%.
\]
We report the mean and median reduction over the test set.

\paragraph{Empty Set Rate.}
The empty set rate is defined as
\[
\frac{1}{|\mathcal{D}_{\mathrm{test}}|}
\sum_{(x,y)\in\mathcal{D}_{\mathrm{test}}}
\mathbb{I}\big\{ \mathcal{C}_\alpha(x) = \emptyset \big\}.
\]
This quantity reflects how often the conformal procedure returns no candidate graph and is reported as a sanity check in finite-candidate settings.

\begin{remark}
    Reporting both mean and median statistics is crucial due to the typically heavy-tailed distribution of candidate set sizes. While both CP and SCQR are guaranteed to satisfy marginal coverage, differences in set size and reduction reflect their relative efficiency, with SCQR expected to yield more adaptive prediction sets.
\end{remark}

\begin{remark}
We distinguish between validity, defined as marginal coverage,
and efficiency, defined as the size of the conformal prediction set.
Since prediction set size is meaningful only when the ground truth
is contained in the conformal set, we report mean and median
\textbf{prediction set size} and \textbf{relative reduction}
conditioned on coverage, i.e., over the subset
\[
\{x \mid (x,y)\in\mathcal D_{\text{test}},\ y\in\mathcal C_\alpha(x)\}.
\]
As a consequence, the reported candidate set statistics in
Table~\ref{tab:cp_scqr_comparison} may differ across methods,
since they are computed only over the subset of inputs for
which the conformal set contains the ground truth.
Because each method achieves coverage on slightly different
instances, the corresponding candidate set averages are
evaluated on different subsets. Coverage itself is always reported over the full test set.
\end{remark}


\newpage
\section{Experimental Details}
\label{app:exp}

\subsection{Datasets}
\label{app:dataset}
We describe here the datasets and evaluation protocols used for the Coloring. 

\paragraph{Coloring (Synthetic).}
For the Coloring task, we use the synthetic dataset introduced in \cite{krzakala2024any2graph}. 
The dataset already provides predefined test and validation splits, in which we used 10k for each. 
We use the test split for conformal calibration and report all evaluation metrics on the validation split.
\begin{figure}[H]
    \centering
   \includegraphics[width=0.6\linewidth]{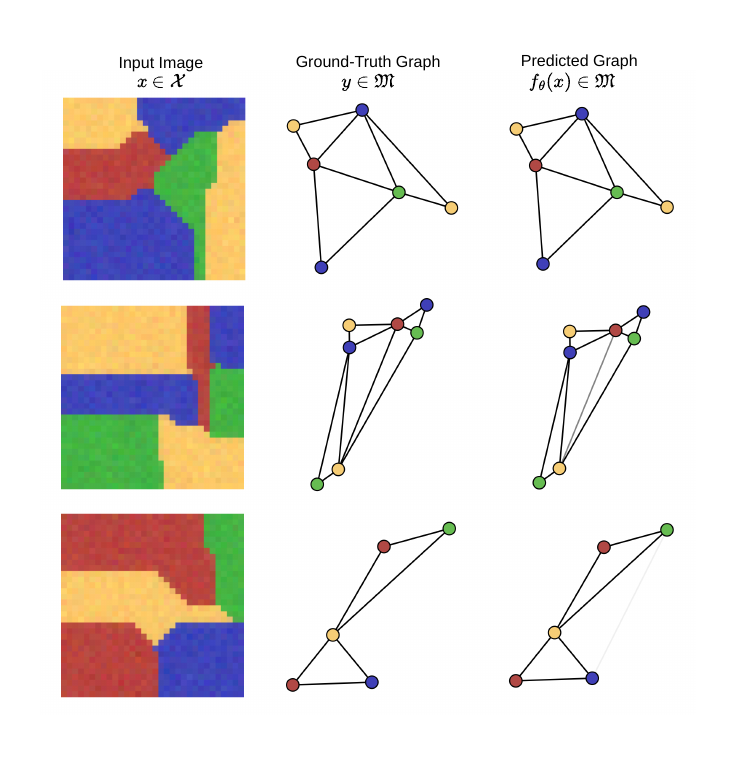}
    \caption{Coloring dataset examples using \textsc{Any2Graph} as graph predictor. }
    \label{fig:placeholder}
\end{figure}

\subsection{Training}

For the Coloring task, we use \textsc{Any2Graph} \citep{krzakala2024any2graph}. 
Both models were trained using the default configurations provided in their respective official repositories, without additional hyperparameter tuning.

Training was performed on NVIDIA A100 GPU. 
Conformal calibration and evaluation were conducted on CPU. 

SCQR calibration based on candidate set size $|\mathcal{L}(x)|$ was done by fitting a 
\href{https://scikit-learn.org/stable/modules/generated/sklearn.linear_model.QuantileRegressor.html}
{linear quantile regressor} to predict calibration set non-conformity scores from
$|\mathcal{L}(x)|$.

SCQR calibration based on \textsc{DreaMS} embeddings was done by training a 2-layer neural network
to predict non-conformity scores from them.

The architecture is as follows:
\begin{itemize}
    \item \texttt{Linear(768, 384)}
    \item \texttt{ReLU}
    \item \texttt{Linear(384, 1)}
\end{itemize}

The network was trained with the following parameters:

\begin{table*}[th]
\centering
\caption{embedding-based SCQR neural network training parameters}
\begin{tabular}{lr}
\toprule
Parameter
& Value \\
\midrule
Optimizer & Adam \\
Learning rate & 0.001 \\
Batch size & 32 \\
Early stopping patience & 3 \\
Early stopping minimum loss improvement & 0.0001 \\
\bottomrule
\end{tabular}
\end{table*}

Complete training on the 10 000 samples of the calibration set takes a few seconds.
\subsection{Execution time}

Execution time is a concern when working with graphs,
especially in conformal prediction where a single prediction
may be compared to thousands of candidates.
We report here the effects of various FGW hyperparameters on computation
time on the Coloring dataset.

The following results were obtained using a Lenovo P16s Gen 4 with an intel ultra 7 chip.

\begin{figure}[H]
    \centering
   \includegraphics[width=0.8\linewidth]{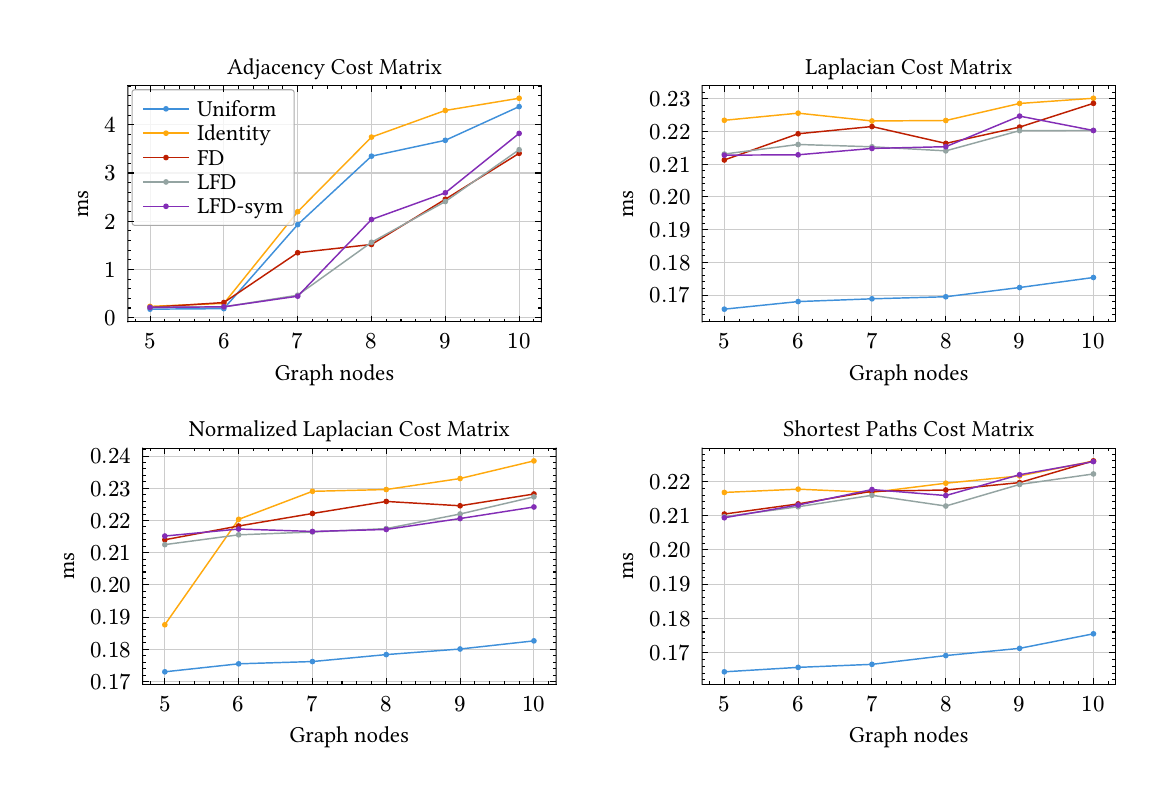}
    \caption{Time taken to compute a FGW distance between a pair of Coloring dataset graphs
    with various amounts of nodes.
    Notice how the performance of uniform initialization is offset from the others,
    given that it is the default initialization of the FGW algorithm, not requiring prior
    computations. The adjacency cost matrix scales very badly with the amount of nodes.
    At 10 nodes per graph, it is already 20 times slower than the other options.}
\end{figure}

\begin{figure}[H]
    \centering
   \includegraphics[width=0.6\linewidth]{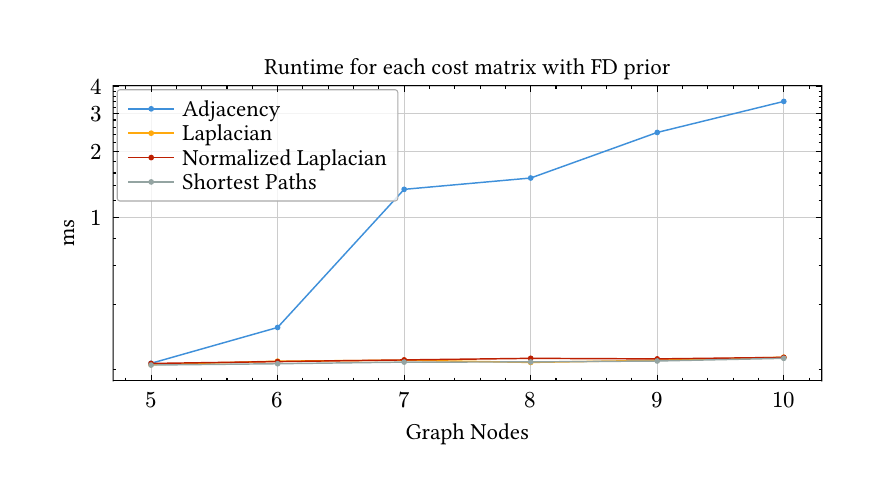}
    \caption{Computation time of the FGW distance between a pair of coloring
    dataset graphs with the FD prior. Although the adjacency cost matrix
    yields good conformal performance, it is not recommended when working with large graphs.}
\end{figure}

\subsection{Extra Results}

\begin{table}[!htbp]
\centering
\caption{Comparison of conformal prediction (CP) and score conformalized quantile 
regression (SCQR) on graph-valued outputs with finite candidate sets. Acc denotes 
the base predictor retrieval accuracy.}
\label{tab:cp_scqr_comparison}
\scriptsize
\begin{tabular}{llcccccccccc}
\toprule
& & & & \multicolumn{2}{c}{\textbf{Conformal Set Size} $\downarrow$} & \multicolumn{2}{c}{\textbf{Candidate Set Size}} & \multicolumn{2}{c}{\textbf{Reduction} \% $\uparrow$} &
\\
\cmidrule(lr){5-6} \cmidrule(lr){7-8} \cmidrule(lr){9-10}
\textbf{Dataset} & \textbf{Model} $f_\theta$
& \textbf{Method}
& \makecell{\textbf{Emp.}\\\textbf{Cov.} (\%)}
& Mean & Median
& Mean & Median
& Mean & Median
& \makecell{\textbf{Empty}\\\textbf{rate} (\%) $\downarrow$} \\
\midrule
\multicolumn{11}{l}{\textit{Coverage $1-\alpha = 90\%$}} \\
\midrule
\multirow{2}{*}{Colors}
& \multirow{2}{*}{\makecell[l]{\textsc{Any2Graph} \\ \tiny{Acc: 82\%}}}
& CP                        & $90.2$ & $4$ & $1$ & $205$ & $223$ & $95.8\%$ & $98.9\%$ & $8.9\%$ \\
& & SCQR $|\mathcal{L}(x)|$ & $90.3$ & $4$ & $1$ & $201$ & $223$ & $95.6\%$ & $98.9\%$ & $8.8\%$ \\
& & SCQR $\operatorname{ResNet}(x)$ & $93.0$ & $4$ & $1$ & $205$ & $218$ & $96.2\%$ & $98.9\%$ & $5.7\%$ \\
\midrule
\multirow{3}{*}{\adduct{$[\text{M+CH}_3\text{COOH-H}]^{-}$}}
& \multirow{3}{*}{\makecell[l]{\textsc{MSAlign} \\ \tiny{Acc: 71.8\%}}}
& CP               & $90.0$ & $59$ & $56$ & $152$ & $159$ & $60.4\%$ & $61.0\%$ & $0.0\%$ \\
& & SCQR $|\mathcal{L}(x)|$ & $90.0$ & $47$ & $48$ & $147$ & $159$ & $64.5\%$ & $64.5\%$ & $0.0\%$ \\
& & SCQR $\textsc{DreaMS}(x)$ & $89.7$ & $29$ & $19$ & $151$ & $159$ & $79.4\%$ & $86.7\%$ & $0.7\%$ \\
\midrule
\multirow{3}{*}{\adduct{$[\text{M+HCOOH-H}]^{-}$}}
& \multirow{3}{*}{\makecell[l]{\textsc{MSAlign} \\ \tiny{Acc: 31.1\%}}}
& CP               & $90.0$ & $166$ & $184$ & $250$ & $256$ & $33.4\%$ & $25.8\%$ & $0.0\%$ \\
& & SCQR $|\mathcal{L}(x)|$ & $90.9$ & $165$ & $184$ & $252$ & $256$ & $34.4\%$ & $25.8\%$ & $0.9\%$ \\
& & SCQR $\textsc{DreaMS}(x)$ & $90.2$ & $143$ & $154$ & $250$ & $256$ & $42.6\%$ & $36.3\%$ & $0.2\%$ \\
\midrule
\multicolumn{11}{l}{\textit{Coverage $1-\alpha = 80\%$}} \\
\midrule
\multirow{3}{*}{\adduct{$[\text{M+CH}_3\text{COOH-H}]^{-}$}}
& \multirow{3}{*}{\makecell[l]{\textsc{EmbCos} \\ \tiny{Acc: 59.9\%}}}
& CP               & $84.7$ & $51.9$ & $53$ & $144.3$ & $133$ & $64.1\%$ & $63.5\%$ & $0.0\%$ \\
& & SCQR $|\mathcal{L}(x)|$ & $84.7$ & $52.1$ & $54$ & $144.3$ & $133$ & $63.9\%$ & $63.5\%$ & $0.0\%$ \\
& & SCQR $\textsc{DreaMS}(x)$ & $79.7$ & $27.8$ & $18$ & $144.1$ & $159$ & $80.0\%$ & $84.3\%$ & $0.0\%$ \\
\midrule
\multirow{3}{*}{\adduct{$[\text{M+K}]^{+}$}}
& \multirow{3}{*}{\makecell[l]{\textsc{EmbCos} \\ \tiny{Acc: 25.0\%}}}
& CP               & $87.2$ & $120.0$ & $95$ & $255.8$ & $256$ & $53.1\%$ & $62.9\%$ & $0.0\%$ \\
& & SCQR $|\mathcal{L}(x)|$ & $86.4$ & $119.0$ & $93$ & $255.8$ & $256$ & $53.5\%$ & $63.7\%$ & $0.0\%$ \\
& & SCQR $\textsc{DreaMS}(x)$ & $83.4$ & $112.8$ & $85$ & $255.8$ & $256$ & $55.9\%$ & $66.8\%$ & $0.0\%$ \\
\midrule
\multirow{3}{*}{\adduct{$[\text{M}]^{+}$}}
& \multirow{3}{*}{\makecell[l]{\textsc{EmbCos} \\ \tiny{Acc: 37.9\%}}}
& CP               & $80.2$ & $124.8$ & $146$ & $250.0$ & $256$ & $49.7\%$ & $43.0\%$ & $0.0\%$ \\
& & SCQR $|\mathcal{L}(x)|$ & $88.8$ & $165.0$ & $151$ & $250.6$ & $256$ & $34.4\%$ & $40.6\%$ & $0.0\%$ \\
& & SCQR $\textsc{DreaMS}(x)$ & $79.3$ & $108.0$ & $144$ & $250.0$ & $256$ & $56.2\%$ & $43.8\%$ & $2.6\%$ \\
\bottomrule
\end{tabular}
\end{table}

\begin{table*}[th]
\centering
\caption{Summary of acronyms used in the experiments.}
\label{tab:cost_matrix_powers}
\begin{tabular}{lll}
\toprule
Category
& Acronym
& Description \\
\midrule
\multirow{4}{*}{Cost Matrix $C$} & $A$ & Adjacency \\
                                 & $L$ & Laplacian \\
                                 & $L^{sym}$ & Laplacian (symmetrically normalized) \\
                                 & $SP$ & Shortest Paths \\

\midrule
\multirow{5}{*}{\makecell[l]{Initial Transport Plan\\ $G_0$}}
& FD & \makecell[l]{\textbf{Feature Diffusion:}\\ Solution to the \textit{Earth Mover's Distance}\\ problem for given nodes of feature augmented vector $(F, AF)$} \\
& LFD & \makecell[l]{\textbf{Laplacian Feature Diffusion:}\\ Same as Feature Diffusion but with features $(F, LF)$} \\
& LFD-sym & \makecell[l]{\textbf{Sym-Norm Laplacian Feature Diffusion:}\\ Same as Feature Diffusion but with features $(F, L^{sym}F)$} \\
& Id & \makecell[l]{\textbf{Identity matrix:} \\ when the two graphs have the same amount of nodes, \\else use the solver's default initialization.} \\
& Uniform & \makecell[l]{\textbf{Uniform:} \\ FGW solver default initialization (uniform transport).} \\


\bottomrule
\end{tabular}
\end{table*}

\subsection{Results by hyperparameter}
\label{app:results}

The choices of the $G_0$ FGW prior transport and of the cost matrix
both influence the speed of convergence and the performance of the produced distances
for conformal set prediction.


\begin{table*}[th]
\centering
\caption{Cost matrix impact on different tasks. The colors task uses the feature diffusion $G_0$ prior.}
\label{tab:cost_matrix_powers}
\begin{tabular}{llcccccc}
\toprule
Task
& Cost Matrix
& \makecell{Empirical \\ Coverage}
& \makecell{Mean\\set size $\downarrow$}
& \makecell{Median\\set size $\downarrow$}
& \makecell{Mean\\reduction (\%) $\uparrow$}
& \makecell{Median\\reduction (\%) $\uparrow$}
& \makecell{Empty\\rate (\%) $\downarrow$} \\
\midrule
\multirow{4}{*}{Colors}
& $A$       & $90.2\%$ & $4$ & $1$ & $95.8\%$ & $98.9\%$ & $9.0\%$ \\
& $L^{sym}$ & $90.4\%$ & $4$ & $1$ & $95.8\%$ & $98.9\%$ & $8.5\%$ \\
& $L$       & $90.8\%$ & $4$ & $1$ & $95.5\%$ & $98.7\%$ & $8.3\%$ \\
& $SP$      & $90.1\%$ & $205$ & $223$ & $2.2\%$ & $0.0\%$ & $2.6\%$ \\
\bottomrule
\end{tabular}
\end{table*}

\begin{table*}[th]
\centering
\caption{$G_0$ prior impact on different tasks. The colors task uses the adjacency cost matrix. Note that uniform $G_0$ prior often converges very slowly.}
\label{tab:cost_matrix_powers}
\begin{tabular}{llcccccc}
\toprule
Task
& $G_0$
& \makecell{Empirical \\ Coverage}
& \makecell{Mean\\set size $\downarrow$}
& \makecell{Median\\set size $\downarrow$}
& \makecell{Mean\\reduction (\%) $\uparrow$}
& \makecell{Median\\reduction (\%) $\uparrow$}
& \makecell{Empty\\rate (\%) $\downarrow$} \\
\midrule

\multirow{5}{*}{Colors}
& LFD     & $90.4\%$ & $4$ & $1$ & $95.8\%$ & $98.9\%$ & $8.7\%$ \\
& FD      & $90.2\%$ & $4$ & $1$ & $95.8\%$ & $98.9\%$ & $9.0\%$ \\
& Uniform & $90.1\%$ & $4$ & $1$ & $95.8\%$ & $98.9\%$ & $8.7\%$ \\
& LFD-sym & $90.0\%$ & $4$ & $1$ & $95.8\%$ & $98.9\%$ & $8.8\%$ \\
& Id      & $90.4\%$ & $6$ & $2$ & $94.4\%$ & $97.9\%$ & $5.1\%$ \\
\bottomrule
\end{tabular}
\end{table*}

\end{document}